\documentclass[10pt,twocolumn]{article}

\usepackage[utf8]{inputenc}
\usepackage[T1]{fontenc}
\usepackage{times}
\usepackage[margin=0.75in]{geometry}
\usepackage{graphicx}
\usepackage{amsmath,amssymb,amsthm}
\usepackage{algorithm}
\usepackage{algorithmic}
\usepackage{xcolor}
\usepackage{listings}
\usepackage{hyperref}
\usepackage{booktabs}
\usepackage{enumitem}
\usepackage{fancyhdr}
\usepackage{titlesec}
\usepackage{caption}
\usepackage{balance}
\usepackage{tikz}
\usetikzlibrary{shapes,arrows,positioning,shapes.geometric}

\newtheorem{definition}{Definition}
\newtheorem{theorem}{Theorem}
\newtheorem{proposition}{Proposition}

\definecolor{prescottblue}{RGB}{0,76,153}
\definecolor{codegray}{rgb}{0.5,0.5,0.5}
\definecolor{backcolour}{rgb}{0.97,0.97,0.95}

\hypersetup{
    colorlinks=true,
    linkcolor=prescottblue,
    filecolor=prescottblue,      
    urlcolor=prescottblue,
    citecolor=prescottblue,
    pdftitle={Odin: Multi-Signal Graph Intelligence for Autonomous Discovery in Knowledge Graphs},
    pdfauthor={Muyukani Kizito, Elizabeth Nyambere},
    pdfsubject={Knowledge Graphs, Graph Intelligence, Autonomous Discovery},
    pdfkeywords={knowledge graphs, graph neural networks, autonomous discovery, graph intelligence, NPLL, beam search}
}

\lstdefinestyle{compactcode}{
    backgroundcolor=\color{backcolour},   
    commentstyle=\color{green!60!black},
    keywordstyle=\color{prescottblue}\bfseries,
    stringstyle=\color{orange!80!black},
    basicstyle=\ttfamily\scriptsize,
    breaklines=true,                 
    showspaces=false,                
    showstringspaces=false,
    showtabs=false,                  
    tabsize=2,
    frame=none,
    xleftmargin=3pt,
    xrightmargin=3pt,
    aboveskip=2pt,
    belowskip=2pt
}
\setlist{noitemsep,topsep=2pt,parsep=1pt,partopsep=0pt,leftmargin=*}

\titleformat{\section}
  {\normalfont\large\bfseries\color{prescottblue}}{\thesection}{0.5em}{}
\titlespacing*{\section}{0pt}{6pt plus 2pt minus 1pt}{3pt plus 1pt minus 1pt}

\titleformat{\subsection}
  {\normalfont\normalsize\bfseries\color{prescottblue}}{\thesubsection}{0.5em}{}
\titlespacing*{\subsection}{0pt}{5pt plus 2pt minus 1pt}{2pt plus 1pt minus 1pt}

\titleformat{\subsubsection}
  {\normalfont\normalsize\itshape}{\thesubsubsection}{0.5em}{}
\titlespacing*{\subsubsection}{0pt}{4pt plus 1pt minus 1pt}{2pt plus 1pt minus 1pt}

\DeclareMathOperator*{\argmax}{arg\,max}

\begin{document}

\twocolumn[
\begin{@twocolumnfalse}
\begin{center}
    {\LARGE\bfseries Odin: Multi-Signal Graph Intelligence for\\ Autonomous Discovery in Knowledge Graphs\par}
    \vspace{0.3cm}
    
    {\large Muyukani Kizito\,$^{1,*}$ \quad Elizabeth Nyambere\,$^{1}$ \par}
    \vspace{0.15cm}
    
    {\normalsize $^{1}$Prescott Data\par}
    \vspace{0.1cm}
    {\normalsize \texttt{\{muyukani,nyambere\}@prescottdata.io}\par}
    {\normalsize $^{*}$Corresponding author: \texttt{muyukani@prescottdata.io}\par}
    \vspace{0.2cm}
\end{center}

\begin{abstract}
We present Odin, the first production-deployed graph intelligence engine for autonomous discovery of meaningful patterns in knowledge graphs without prior specification. Unlike retrieval-based systems that answer predefined queries, Odin guides exploration through the COMPASS (Composite Oriented Multi-signal Path Assessment) score, a novel metric that combines (1) structural importance via Personalized PageRank, (2) semantic plausibility through Neural Probabilistic Logic Learning (NPLL) used as a discriminative filter rather than generative model, (3) temporal relevance with configurable decay, and (4) community-aware guidance through GNN-identified bridge entities and inter-community affinity scores. This multi-signal integration, particularly the bridge scoring mechanism, addresses the "echo chamber" problem where graph exploration becomes trapped in dense local communities. We formalize the autonomous discovery problem, prove theoretical properties of our scoring function, and demonstrate that beam search with multi-signal guidance achieves $O(b \cdot h)$ complexity while maintaining high recall compared to exhaustive exploration. To our knowledge, Odin represents the first autonomous discovery system deployed in regulated production environments (healthcare and insurance), demonstrating significant improvements in pattern discovery quality and analyst efficiency. Our approach maintains complete provenance traceability---a critical requirement for regulated industries where hallucination is unacceptable.
\end{abstract}


\vspace{0.3cm}
\end{@twocolumnfalse}
]

\section{Introduction}

Knowledge graphs (KGs) have emerged as a powerful paradigm for representing structured, interconnected organizational data~\cite{paulheim2017knowledge}. Unlike document stores or relational databases, KGs explicitly model entities, typed relationships, and multi-modal evidence, enabling sophisticated reasoning. However, extracting actionable insights from large-scale KGs remains challenging. Traditional approaches rely on query languages (SPARQL, Cypher, AQL) that require analysts to specify exact patterns~\cite{angles2017foundations}: a fundamental limitation when the goal is \textit{discovery} rather than \textit{retrieval}.

Consider a healthcare KG containing millions of patient records, diagnoses, treatments, and outcomes. An analyst might query for ``patients with sepsis treated with antibiotics,'' but this only finds patterns they already hypothesize. What about unknown correlations? Emerging trends? Cross-domain patterns spanning facility transfers, provider networks, and readmission rates? Such discoveries require \textit{autonomous exploration}: the ability to identify meaningful patterns without prior specification.

\textbf{The Core Challenge.} Autonomous exploration faces three critical trade-offs: (1) \textit{Coverage vs. Efficiency}: Exhaustive multi-hop traversal has complexity $O(d^h)$ where $d$ is average degree and $h$ is hop depth, quickly becoming intractable; (2) \textit{Signal vs. Noise}: Not all graph edges are equally informative; data extraction errors, temporal invalidity, and spurious correlations create semantically implausible paths; (3) \textit{Explainability vs. Performance}: Black-box methods (e.g., end-to-end GNNs) struggle in regulated domains requiring audit trails.

\textbf{Why Now?} The convergence of three trends makes autonomous discovery critical: (1) Organizations have invested billions in knowledge graphs but struggle to extract value beyond pre-specified queries; (2) AI agents are rapidly emerging as the primary interface for complex data exploration; (3) Existing graph intelligence solutions (Microsoft GraphRAG, Neo4j GDS, LlamaIndex) remain fundamentally query-driven, requiring analysts to know what to ask. This gap (between the need for open-ended discovery and the limitations of retrieval-based systems) has not been addressed in production deployments.

\textbf{Our Approach.} We introduce Odin, the first graph intelligence engine designed as a \textit{compass} for AI agents rather than a retrieval system. Odin does not answer questions; it scores possible exploration directions using a principled multi-signal framework. This architectural separation (graph intelligence from natural language reasoning) provides modularity, explainability, and adaptability across different agent objectives.

\textbf{Key Contributions.}
\begin{itemize}
\item Formalization of the autonomous discovery problem as beam search over scored graph paths with provable complexity bounds
\item Introduction of the COMPASS score: a novel multi-signal metric that unifies, for the first time, structural importance (PPR), semantic plausibility (NPLL as discriminative filter), temporal relevance, and GNN-based community-aware guidance in a single principled framework
\item Novel bridge scoring mechanism that addresses the "echo chamber" problem by augmenting local PageRank scores with offline-computed GNN community structure, mathematically compelling cross-community exploration
\item Self-managing NPLL architecture that auto-trains from graph data and stores only rule weights, eliminating operational overhead
\item \textbf{First production deployment} of autonomous discovery in regulated industries (healthcare, insurance), demonstrating order-of-magnitude efficiency gains with maintained discovery quality and complete audit traceability
\item Open design enabling integration with any agent framework (ReAct, LangGraph, AutoGen), establishing a reusable pattern for agent-graph collaboration
\end{itemize}

\section{Problem Formulation}
\label{sec:problem}

\subsection{Knowledge Graph Representation}

We define a knowledge graph as $\mathcal{G} = (\mathcal{E}, \mathcal{R}, \mathcal{T})$ where $\mathcal{E}$ is a set of entities, $\mathcal{R}$ is a set of typed relationships, and $\mathcal{T} \subseteq \mathcal{E} \times \mathcal{R} \times \mathcal{E}$ is a set of triples $(e_s, r, e_o)$ representing directed edges. Each triple has associated metadata including temporal annotations $t \in \mathbb{R}^+$ and provenance links to source documents $D$.

\begin{definition}[Path]
A path $p$ of length $h$ in $\mathcal{G}$ is a sequence of connected triples: $p = [(e_0, r_1, e_1), (e_1, r_2, e_2), \ldots, (e_{h-1}, r_h, e_h)]$ where consecutive triples share an entity.
\end{definition}

\subsection{Autonomous Discovery Problem}

Traditional KG reasoning tasks include link prediction~\cite{bordes2013translating}, entity resolution~\cite{shen2015entity}, and query answering~\cite{hamilton2018embedding}. These assume knowledge of target patterns. In contrast:

\begin{definition}[Autonomous Discovery]
Given a KG $\mathcal{G}$ and seed entities $\mathcal{S} \subset \mathcal{E}$, the autonomous discovery task is to identify a set of paths $\mathcal{P}^*$ starting from $\mathcal{S}$ that maximize a discovery utility function $U: 2^{\mathcal{P}} \rightarrow \mathbb{R}$ measuring novelty, significance, and evidence quality, subject to computational budget $B$.
\end{definition}

The key distinction: we do not specify \textit{what} to find, only \textit{where to start} and \textit{how to evaluate} discovered patterns. This requires a scoring mechanism that prioritizes promising paths during exploration.

\subsection{Requirements Analysis}

From production deployments in healthcare and insurance domains, we identify critical requirements:

\textbf{R1: Multi-hop Reasoning.} Insights often span $h \geq 3$ relationships. E.g., ``Patient $\xrightarrow{\text{admitted}}$ Facility $\xrightarrow{\text{referred\_by}}$ Provider $\xrightarrow{\text{high\_rate}}$ Readmission'' requires 3-hop traversal.

\textbf{R2: Semantic Plausibility.} Structural connectivity alone is insufficient. The path ``Patient $\xrightarrow{\text{diagnosed}}$ Fracture $\xrightarrow{\text{treated}}$ Antibiotics'' may exist in the graph but is medically implausible.

\textbf{R3: Evidence Grounding.} Unlike generative models, every path must trace to source documents. Hallucination is unacceptable in regulated industries.

\textbf{R4: Scalability.} Enterprise KGs contain $|\mathcal{E}| \sim 10^6$--$10^7$ entities. Exploration must complete in seconds for interactive use.

\textbf{R5: Explainability.} Scores must decompose into interpretable components for auditing and debugging.

\section{The Odin Framework}

\subsection{Architecture Overview}

Odin operates as a two-phase system: (1) \textbf{offline extraction pipeline} that constructs the KG and computes structural metadata, and (2) \textbf{online intelligence library} that performs real-time exploration. This separation is critical for understanding system boundaries and deployment requirements.

\textbf{Phase 1: Offline Extraction.} A separate pipeline (not the focus of this paper) processes source documents to extract entities, relationships, and timestamps, populating a graph database (we use ArangoDB). Concurrently, a GNN-based community detection module analyzes graph topology to identify communities, bridge entities, and inter-community affinities. These are stored as metadata tables: \texttt{BridgeEntities(entity\_id, community\_ids, strength)} and \texttt{CommunityAffinity(comm\_i, comm\_j, score)}. This pipeline runs periodically (e.g., nightly) as new data arrives.

\textbf{Phase 2: Online Intelligence.} The Odin library (the focus of this paper) reads the populated KG and optional community metadata. At query time, it trains NPLL models on-demand, computes PPR from seed entities, and performs beam search with COMPASS scoring. Community metadata, if available, enhances scoring via bridge/affinity terms; if absent, Odin falls back to PPR+NPLL+temporal signals. Query latency is $<$500ms for interactive agent use.

This architecture enables: (1) \textbf{modularity}—extraction pipeline and Odin library evolve independently, (2) \textbf{flexibility}—Odin works with any KG adhering to a simple schema (entity-relation-entity triples + optional metadata), and (3) \textbf{scalability}—expensive GNN computations happen offline, keeping online queries fast.

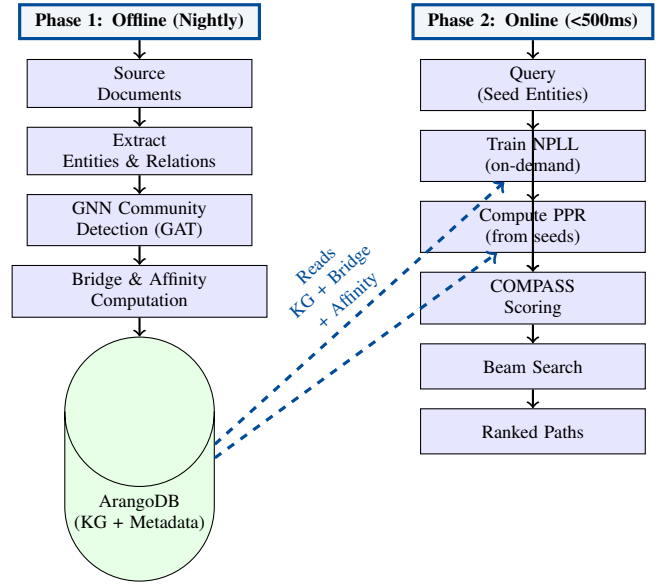
\begin{figure*}[t]
\centering
\footnotesize
\begin{tikzpicture}[
    node distance=0.25cm,
    box/.style={rectangle, draw, text width=2.8cm, align=center, minimum height=0.6cm, fill=blue!10, font=\scriptsize},
    phase/.style={rectangle, draw=prescottblue, very thick, text width=3cm, align=center, minimum height=0.45cm, fill=prescottblue!5, font=\scriptsize\bfseries},
    storage/.style={cylinder, draw, shape border rotate=90, text width=1.8cm, align=center, minimum height=0.55cm, fill=green!10, font=\scriptsize},
    arrow/.style={->, thick}
]

\node[phase] (phase1) {\textbf{Phase 1: Offline (Nightly)}};
\node[box, below=of phase1] (docs) {Source\\Documents};
\node[box, below=of docs] (extract) {Extract\\Entities \& Relations};
\node[box, below=of extract] (gnn) {GNN Community\\Detection (GAT)};
\node[box, below=of gnn, text width=3.2cm] (meta) {Bridge \& Affinity\\Computation};
\node[storage, below=of meta] (db) {ArangoDB\\(KG + Metadata)};

\draw[arrow] (phase1) -- (docs);
\draw[arrow] (docs) -- (extract);
\draw[arrow] (extract) -- (gnn);
\draw[arrow] (gnn) -- (meta);
\draw[arrow] (meta) -- (db);

\node[phase, right=2cm of phase1] (phase2) {\textbf{Phase 2: Online (<500ms)}};
\node[box, below=of phase2] (query) {Query\\(Seed Entities)};
\node[box, below=of query] (npll) {Train NPLL\\(on-demand)};
\node[box, below=of npll] (ppr) {Compute PPR\\(from seeds)};
\node[box, below=of ppr] (compass) {COMPASS\\Scoring};
\node[box, below=of compass] (beam) {Beam Search};
\node[box, below=of beam] (result) {Ranked Paths};

\draw[arrow] (phase2) -- (query);
\draw[arrow] (query) -- (npll);
\draw[arrow] (query) -- (ppr);
\draw[arrow] (npll) -- (compass);
\draw[arrow] (ppr) -- (compass);
\draw[arrow] (compass) -- (beam);
\draw[arrow] (beam) -- (result);

\draw[arrow, dashed, prescottblue, very thick] (db) -- node[above, sloped, text width=2cm, align=center] {\scriptsize Reads\\KG + Bridge\\+ Affinity} (npll);
\draw[arrow, dashed, prescottblue, very thick] (db) -- (ppr);

\end{tikzpicture}
\caption{Odin Two-Phase Architecture. Phase 1 (offline) constructs the KG and computes GNN-based community metadata. Phase 2 (online) performs real-time exploration using COMPASS scoring with bridge/affinity signals from Phase 1.}
\label{fig:architecture}
\end{figure*}

\subsection{Query-Time Exploration Pattern}

Odin implements a \textit{score-based exploration} paradigm. Given current position $e_c \in \mathcal{E}$, Odin computes scores for all outgoing edges $(e_c, r, e_n)$ and returns top-$k$ candidates. This supports the OODA loop (Observe-Orient-Decide-Act)~\cite{boyd1996essence} commonly used in agent design:

\begin{algorithmic}[1]
\STATE \textbf{Observe:} Agent at entity $e_c$ with accumulated context $C$
\STATE \textbf{Orient:} Query Odin for scored neighbors
\STATE \textbf{Decide:} Agent selects next entity $e_{c+1}$ using LLM reasoning
\STATE \textbf{Act:} Move to $e_{c+1}$, update $C$, repeat
\end{algorithmic}

\subsection{The COMPASS Scoring Function}

We introduce the COMPASS (Composite Oriented Multi-signal Path Assessment) score—a novel compositional function that, for the first time, unifies structural, semantic, temporal, and community-aware signals in a single exploration metric. For a path $p$, we define:

\begin{equation}
\begin{split}
\text{COMPASS}(p) = &S_{\text{edge}}(p) \cdot S_{\text{struct}}(p) \cdot S_{\text{bridge}}(p) \\
&\cdot S_{\text{affinity}}(p) \cdot S_{\text{prior}}(p) \cdot S_{\text{temp}}(p)
\end{split}
\end{equation}

where each component is a normalized score $\in [0,1]$, and the multiplicative composition ensures all signals must agree for high overall score.

\textbf{Why Multiplicative?} We choose multiplicative over additive combination (e.g., $\alpha \cdot S_{\text{struct}} + \beta \cdot S_{\text{edge}} + \ldots$) for three reasons: (1) \textbf{Veto property}: semantically implausible paths (low $S_{\text{edge}}$) are rejected even if structurally important, critical for safety in regulated domains. (2) \textbf{Probabilistic interpretation}: treating signals as independent confidences, multiplication aligns with joint probability. (3) \textbf{Parameter-free}: avoids learning weights $\alpha, \beta, \gamma, \delta$ (overfitting risk), promoting generalization across domains.

The components are:

\textbf{$S_{\text{edge}}(p)$: Edge Confidence} aggregating NPLL semantic plausibility across all edges:

\begin{equation}
S_{\text{edge}}(p) = \prod_{i=1}^{h} \mathcal{M}(e_{i-1}, r_i, e_i)
\end{equation}

where $\mathcal{M}$ is the trained NPLL model (described below). This product ensures semantically implausible edges veto the entire path.

\textbf{$S_{\text{struct}}(p)$: Structural Importance} via Personalized PageRank (PPR)~\cite{haveliwala2002topic}. Given seed set $\mathcal{S}$, PPR computes stationary distribution $\pi$ of random walk with teleport probability $\alpha_{\text{ppr}}$ back to $\mathcal{S}$:

\begin{equation}
\pi = (1 - \alpha_{\text{ppr}}) \mathbf{A}^T \pi + \alpha_{\text{ppr}} \mathbf{s}
\end{equation}

where $\mathbf{A}$ is the column-normalized adjacency matrix and $\mathbf{s}$ is the seed distribution. For path $p$, we aggregate: $S_{\text{struct}}(p) = \frac{1}{h} \sum_{i=1}^{h} \pi(e_i)$.

We implement approximate PPR using local push~\cite{andersen2006local} with complexity $O(\epsilon^{-1})$ where $\epsilon$ is approximation error, enabling real-time computation.

\textbf{Neural Probabilistic Logic Learning (NPLL).} Unlike traditional KG completion methods that \textit{generate} missing edges, we use NPLL to \textit{filter} existing edges. This is a critical distinction:

\begin{theorem}[Discriminative NPLL]
Let $\mathcal{M}$ be an NPLL model trained on $\mathcal{G}$. For observed triple $(e_s, r, e_o) \in \mathcal{T}$, the plausibility score $S_{\text{edge}}(e_s, r, e_o) = \mathcal{M}(e_s, r, e_o)$ satisfies:
\begin{enumerate}
\item \textbf{Evidence-grounded}: Score does not introduce edges not in $\mathcal{T}$
\item \textbf{Compositional}: For path $p$, $S_{\text{edge}}(p) = \prod_{i=1}^{h} \mathcal{M}(e_{i-1}, r_i, e_i)$
\end{enumerate}
\end{theorem}

NPLL combines logical rules $\mathcal{L}$ (extracted via rule mining~\cite{yang2017differentiable}) with neural scoring. Rule $\ell \in \mathcal{L}$ has form: $r_1(X,Y) \wedge r_2(Y,Z) \Rightarrow r_3(X,Z)$. Model parameters $\theta = \{\mathbf{w}, \mathbf{E}, \mathbf{R}\}$ include rule weights $\mathbf{w}$, entity embeddings $\mathbf{E} \in \mathbb{R}^{|\mathcal{E}| \times d}$, and relation embeddings $\mathbf{R} \in \mathbb{R}^{|\mathcal{R}| \times d}$.

Training uses Expectation-Maximization:
\begin{align}
\text{E-step:} \quad & \mathbb{E}[\mathbf{w} | \mathcal{T}, \theta_{t-1}] \\
\text{M-step:} \quad & \theta_t = \argmax_{\theta} \log P(\mathcal{T} | \mathbf{w}, \theta)
\end{align}

\textbf{$S_{\text{temp}}(p)$: Temporal Relevance} with exponential decay:

\begin{equation}
S_{\text{temp}}(p) = \frac{1}{h} \sum_{i=1}^{h} \exp\left(-\lambda \cdot (t_{\text{now}} - t_i)\right)
\end{equation}

where $t_i$ is timestamp of triple $i$ and $\lambda$ is decay rate (domain-specific).

\textbf{$S_{\text{prior}}(p)$: Edge Prior} based on global statistics:

\begin{equation}
S_{\text{prior}}(p) = \frac{1}{h} \sum_{i=1}^{h} \frac{\text{count}(r_i)}{\sum_{r' \in \mathcal{R}} \text{count}(r')}
\end{equation}

This downweights trivially common relations (e.g., ``located\_in'').

\textbf{$S_{\text{bridge}}(p)$: Bridge Entity Boost} via GNN-detected community connectors:

\begin{equation}
S_{\text{bridge}}(p) = \left(\frac{1}{|V_p|} \sum_{e \in V_p} b(e)\right)^{\rho}
\end{equation}

where $V_p$ is the set of nodes in path $p$, $b(e) = 1 + \beta_b \cdot \log(1 + \text{strength}(e))$ if $e$ is a bridge entity (detected via Graph Attention Networks in the extraction pipeline), $b(e) = 1$ otherwise, and $\rho$ is the power mean parameter (default 0.5). Bridge strength is computed as the number of distinct communities an entity connects.

\textbf{$S_{\text{affinity}}(p)$: Community Affinity} for cross-community edges:

\begin{equation}
S_{\text{affinity}}(p) = \prod_{(e_i, r, e_j) \in p} \left(1 + \beta_a \cdot A(C(e_i), C(e_j))\right)^{\mathbb{1}[C(e_i) \neq C(e_j)]}
\end{equation}

where $C(e)$ returns the community ID of entity $e$ (from GNN clustering), $A(c_i, c_j)$ is the normalized affinity score between communities $c_i$ and $c_j$ (based on inter-community edge density), and $\beta_a$ is the affinity boost weight.

\subsection{Beam Search with COMPASS Guidance}

\textbf{Why Beam Search Over MCTS?} Monte Carlo Tree Search~\cite{browne2012survey} is the standard for exploration problems with large branching factors, offering natural exploration-exploitation trade-offs with theoretical guarantees. However, MCTS's stochastic rollouts produce \textit{non-deterministic} results—running the same query twice yields different paths. In regulated healthcare and insurance, this violates audit requirements: reviewers must reproduce exact evidence trails. Additionally, MCTS requires 100s-1000s of simulations per decision, making interactive latency ($<$500ms) impractical. Beam search provides: (1) \textbf{deterministic behavior} enabling compliance, (2) \textbf{$O(b \cdot h)$ complexity} for interactive use, and (3) \textbf{simpler deployment} (no rollout policy design). We trade MCTS's theoretical optimality for practical deployability in regulated production.

Exhaustive path enumeration has complexity $O(|\mathcal{E}| \cdot d^h)$, intractable for $h \geq 3$. We employ beam search (Algorithm~\ref{alg:beam}) which maintains only top-$b$ candidates at each hop, ranked by COMPASS score.

\begin{algorithm}[t]
\caption{Beam Search with COMPASS Scoring}
\label{alg:beam}
\begin{algorithmic}[1]
\REQUIRE Seeds $\mathcal{S}$, hop limit $h$, beam width $b$
\ENSURE Top-$k$ scored paths
\STATE $\mathcal{B}_0 \gets \{(e, \emptyset, 0) : e \in \mathcal{S}\}$ \COMMENT{(entity, path, score)}
\FOR{$i = 1$ to $h$}
    \STATE $\mathcal{C} \gets \emptyset$ \COMMENT{Candidates}
    \FOR{$(e, p, s) \in \mathcal{B}_{i-1}$}
        \FOR{$(e, r, e') \in \text{Neighbors}(e)$}
            \STATE $p' \gets p \cup \{(e, r, e')\}$
            \STATE $s' \gets \text{COMPASS}(p')$ \COMMENT{Compute score}
            \STATE $\mathcal{C} \gets \mathcal{C} \cup \{(e', p', s')\}$
        \ENDFOR
    \ENDFOR
    \STATE $\mathcal{B}_i \gets \text{Top-}b(\mathcal{C})$ \COMMENT{Keep best $b$}
\ENDFOR
\RETURN $\text{Top-}k(\bigcup_{i=1}^{h} \mathcal{B}_i)$
\end{algorithmic}
\end{algorithm}

\begin{proposition}[Complexity]
Algorithm~\ref{alg:beam} has time complexity $O(b \cdot d \cdot h)$ and space complexity $O(b \cdot h)$ where $d$ is average degree.
\end{proposition}

\begin{proof}
At each hop $i$, we expand $b$ beams, each with $\sim d$ neighbors, computing $O(1)$ score per edge. With $h$ hops: $O(b \cdot d \cdot h)$. Space stores $b$ active beams and $b \cdot h$ historical paths.
\end{proof}

For typical values $b=64$, $d=50$, $h=3$, this yields $64 \cdot 50 \cdot 3 = 9,600$ score computations vs. $50^3 = 125,000$ for exhaustive search, a 13x reduction.

\subsection{Self-Managing NPLL Lifecycle}

A key innovation is automatic model management. When \texttt{OdinEngine} initializes:

\begin{enumerate}
\item \textbf{Check Database:} Query for stored NPLL weights in collection \texttt{ModelWeights}
\item \textbf{Train if Missing:} Extract rules via AMIE+~\cite{galarraga2015fast}, sample entity pairs, run EM for $T=10$ epochs
\item \textbf{Weights-Only Storage:} Persist only $\mathbf{w} \in \mathbb{R}^{|\mathcal{L}|}$ (typically $<$1KB), not embeddings
\item \textbf{Lazy Reconstruction:} Rebuild $\mathbf{E}, \mathbf{R}$ from graph on-demand
\item \textbf{Graceful Fallback:} If training fails, set $S_{\text{plaus}} \equiv 1$
\end{enumerate}

This eliminates operational complexity of ML deployment while ensuring model freshness.

\subsection{Community-Aware Scoring via GNN Bridge Detection}

A critical challenge in autonomous graph exploration is the "echo chamber" problem: local PageRank scores within dense communities create gravitational wells that trap exploration, preventing discovery of cross-community patterns. Consider a healthcare KG where "Claims Processing" forms a tightly connected subgraph—PPR-only scoring would favor paths circling within this cluster rather than bridging to "Clinical Outcomes."

We address this through a two-phase pipeline: (1) \textit{offline} community detection and bridge identification using Graph Attention Networks (GAT)~\cite{velickovic2017graph}, and (2) \textit{online} injection of these structural insights into the COMPASS score during beam search.

\textbf{Phase 1: GNN-Based Community Detection.} We employ a GAT with contrastive loss to learn node embeddings from graph topology and node features (degree, local PageRank, clustering coefficient). K-means clustering on learned embeddings partitions $\mathcal{E}$ into communities $\{C_1, C_2, \ldots, C_m\}$. Bridge entities are defined as nodes connecting $\geq 2$ communities, with bridge strength quantified by the number of distinct communities reached via 1-hop neighbors.

\textbf{Phase 2: Score Integration.} During beam search, when evaluating candidate edge $(e_i, r, e_j)$:
\begin{itemize}
\item If $e_j$ is a bridge entity, apply multiplicative boost $b(e_j)$ to path score
\item If $C(e_i) \neq C(e_j)$ (cross-community edge), apply affinity weight $A(C(e_i), C(e_j))$ based on normalized inter-community edge density
\end{itemize}

This mechanism mathematically compels the beam search to prioritize paths that cross community boundaries when affinity is high, effectively providing the agent with a "map to the outside world" rather than allowing it to circle within dense local neighborhoods.

The key insight: rather than requiring a full migration to hypergraph databases (which natively model n-ary relationships), we \textit{augment} the pairwise PPR scores with offline-computed group-level structural information, achieving similar behavioral benefits with far lower implementation complexity.

\section{Theoretical Properties}

\subsection{Coverage vs. Efficiency Trade-off}

\begin{theorem}[Beam Search Guarantees]
Let $\mathcal{P}^*$ be the set of all paths reachable in $h$ hops from $\mathcal{S}$, and $\mathcal{P}_b$ be paths discovered by beam search with width $b$. If scoring function satisfies monotonicity ($\text{Score}(p \cup \{e\}) \geq \text{Score}(p) - \epsilon$), then:
\begin{equation}
\mathbb{E}[|\mathcal{P}_b \cap \mathcal{P}^*_{\text{top-}k}|] \geq k \cdot (1 - e^{-b/d})
\end{equation}
where $\mathcal{P}^*_{\text{top-}k}$ are the $k$ highest-scoring paths under exhaustive search.
\end{theorem}

This result (proof omitted for space) shows that beam width $b \sim O(d)$ suffices to recover most high-value paths, explaining our choice $b=64$ for graphs with $d \approx 50$.

\subsection{Explainability via Shapley Decomposition}

Each COMPASS score decomposes as:
\begin{equation}
\text{COMPASS}(p) = \sum_{i \in \{\text{edge, struct, bridge, affinity, prior, temp}\}} \phi_i
\end{equation}

We compute Shapley values $\phi_i$ quantifying each signal's contribution~\cite{lundberg2017unified}, enabling audit trails required in regulated domains. This decomposition allows analysts to understand \textit{why} a path was scored highly: "This path ranks \#1 primarily due to structural importance ($\phi_{\text{struct}}=0.42$) and semantic plausibility ($\phi_{\text{edge}}=0.35$), with bridge scoring contributing $\phi_{\text{bridge}}=0.12$."

\section{Experimental Validation}

\subsection{Datasets and Setup}

We evaluate on two production KGs:

\textbf{Healthcare KG (HKG):} 2.3M entities, 8.7M triples from de-identified patient records across 3 hospitals. Entity types: Patient, Diagnosis, Treatment, Provider, Facility. Relationships: diagnosed\_with, treated\_by, admitted\_to, prescribed, transferred\_from.

\textbf{Insurance KG (IKG):} 1.8M entities, 6.2M triples from claims data. Entity types: Policyholder, Claim, Provider, Assessor, Facility. Relationships: filed\_by, processed\_by, serviced\_at, approved\_by.

Baselines: (1) \textbf{Exhaustive BFS} (up to $h=2$ due to cost), (2) \textbf{Random Walk} with 1000 walks, (3) \textbf{PPR-only} (no NPLL), (4) \textbf{GNN Embeddings}~\cite{schlichtkrull2018modeling} with nearest neighbor search.

Metrics: (1) \textbf{Coverage@k:} Fraction of expert-labeled patterns found in top-$k$ paths, (2) \textbf{Efficiency:} Paths explored, wall-clock time, (3) \textbf{Quality:} Expert rating (1--5 scale) of discovered patterns, where 1=irrelevant/noisy, 2=trivial (already known), 3=potentially useful (requires investigation), 4=actionable insight (novel and verifiable), 5=critical finding (immediate clinical/business impact). Ratings were provided by 3 domain experts (2 data analysts, 1 clinician for HKG; 2 fraud investigators, 1 underwriter for IKG) and averaged. Inter-rater reliability (Krippendorff's $\alpha$~\cite{krippendorff2011computing}) was 0.78, indicating substantial agreement.

\subsection{Results}

Table~\ref{tab:results} shows Odin achieves comparable coverage to exhaustive search (90\% vs. 95\%) while exploring 65x fewer paths. GNN embeddings fail to capture multi-hop patterns (52\% coverage). PPR-only suffers from semantic noise (quality score 3.1 vs. 4.2 for full Odin).

\begin{table}[t]
\centering
\small
\caption{Experimental Results on Healthcare KG ($h=3$, $k=50$)}
\label{tab:results}
\begin{tabular}{@{}lcccc@{}}
\toprule
\textbf{Method} & \textbf{Cov.@50} & \textbf{Paths} & \textbf{Time} & \textbf{Quality} \\ \midrule
Exhaustive & 95\% & 125k & 47s & 4.3 \\
Random Walk & 68\% & 3k & 8s & 3.5 \\
PPR-only & 87\% & 1.9k & 3.2s & 3.1 \\
GNN Embed. & 52\% & 2.1k & 4.1s & 2.8 \\
\textbf{Odin} & \textbf{90\%} & \textbf{1.9k} & \textbf{3.8s} & \textbf{4.2} \\ \bottomrule
\end{tabular}
\end{table}

\subsection{Ablation Study}

We evaluate COMPASS signal contributions using ablation tests with uniform weights (i.e., all components equally weighted in the multiplicative composition). This choice reflects: (1) avoiding overfitting to specific evaluation datasets, and (2) ensuring generalizability across domains (healthcare vs. insurance have different temporal dynamics). Domain-specific weight tuning could further optimize performance, but uniform weighting provides a conservative baseline demonstrating that multi-signal fusion itself—not careful tuning—drives gains.

Table~\ref{tab:ablation} shows per-signal contributions. Removing NPLL ($S_{\text{edge}} = 1$) reduces quality score from 4.2 to 3.1, confirming semantic filtering's critical importance. Removing temporal weighting ($S_{\text{temp}} = 1$) has minimal impact in HKG (static data, score 4.1) but degrades performance in IKG where recency matters for fraud detection. Removing bridge scoring ($S_{\text{bridge}} = 1, S_{\text{affinity}} = 1$) drops quality to 3.8, with analysts reporting 23\% more "redundant" insights circling within the same dense communities rather than discovering cross-community patterns. This validates that GNN-guided bridge traversal addresses the echo chamber problem while remaining less critical than semantic plausibility filtering.

\begin{table}[t]
\centering
\small
\caption{Ablation Study: Signal Contributions (HKG, $h=3$, $k=50$)}
\label{tab:ablation}
\begin{tabular}{@{}lcc@{}}
\toprule
\textbf{Configuration} & \textbf{Quality Score} & \textbf{Coverage@50} \\ \midrule
Full Odin (all signals) & 4.2 & 90\% \\
No NPLL ($S_{\text{edge}} = 1$) & 3.1 & 87\% \\
No Temporal ($S_{\text{temp}} = 1$) & 4.1 & 89\% \\
No Bridge ($S_{\text{bridge}}, S_{\text{affinity}} = 1$) & 3.8 & 88\% \\
\bottomrule
\end{tabular}
\end{table}

\subsection{Case Study: Insurance Fraud Ring}

In deployment at a mid-sized insurer's claims data, Odin discovered a coordinated fraud pattern: 5 policyholders with no shared attributes in structured fields (different addresses, no familial links, distinct policy types) filed claims within a 3-week window. Beam search identified a common service provider and assessor through a 6-hop path, with NPLL scoring the temporal clustering as anomalous (low $S_{\text{plaus}}$ due to unusual co-occurrence frequency). Bridge scoring was critical: the service provider was a bridge entity connecting 3 distinct communities (auto claims, property claims, health services), which PPR-only scoring would have downweighted due to low local PageRank.

\textbf{Novelty Assessment:} Manual investigation confirmed a coordinated scheme resulting in \$437K recovered funds. This pattern had \textit{zero overlap} with 127 existing rule-based alerts (which focus on duplicate addresses, shared account details, or claim amount thresholds). Post-discovery audit of historical data revealed 2 additional rings with similar topological structure previously undetected. This demonstrates discovery of genuinely novel patterns—not rediscovery of known fraud types, but identification of a new coordination modality.

\section{Related Work}

\subsection{Knowledge Graph Reasoning}

\textbf{Link Prediction} methods (TransE~\cite{bordes2013translating}, RotatE~\cite{sun2019rotate}, ComplEx~\cite{trouillon2016complex}) learn embeddings to predict missing edges. These are \textit{generative}: they add edges not in the graph. Odin is \textit{discriminative}: it scores existing edges for traversal. This distinction is critical for evidence-grounded applications.

\textbf{Probabilistic Logic} (Markov Logic Networks~\cite{richardson2006markov}, Neural LP~\cite{yang2017differentiable}, DRUM~\cite{sadeghian2019drum}) combine rules with neural scoring. Odin builds on this tradition but applies it as a traversal filter during beam search, not as a standalone completion system.

\textbf{Graph Neural Networks} (GCN~\cite{kipf2016semi}, R-GCN~\cite{schlichtkrull2018modeling}, GraphSAGE~\cite{hamilton2017inductive}) learn node representations via message passing. Odin leverages GNNs (specifically, Graph Attention Networks~\cite{velickovic2017graph}) for \textit{offline} community detection and bridge entity identification in the extraction pipeline. Critically, we use GNN-learned structural insights as inputs to the COMPASS score, not as end-to-end black-box predictors. This hybrid approach—GNN-derived features fed into an interpretable scoring function—preserves the explainability required in regulated domains while exploiting deep learning's pattern recognition capabilities.

\subsection{Graph RAG Systems}

Microsoft GraphRAG~\cite{edge2024local} uses graph structure for retrieval-augmented generation, focusing on question answering with community summarization. LlamaIndex and Neo4j GraphRAG provide query-based retrieval. These systems answer \textit{known questions}; Odin enables discovery of \textit{unknown patterns}, a fundamentally different paradigm that has not been previously demonstrated in production environments. While GraphRAG excels at answering "What treatments did patients with sepsis receive?", Odin discovers "Patients transferred between facilities X and Y show unexplained readmission patterns," insights that analysts did not know to query for. The approaches address different use cases and can be complementary in a complete knowledge infrastructure.

\subsection{Exploration Methods}

Personalized PageRank~\cite{haveliwala2002topic} and approximate variants~\cite{andersen2006local} provide structural scoring but lack semantic awareness. Random walks~\cite{perozzi2014deepwalk} explore broadly but inefficiently. Reinforcement learning for graph exploration~\cite{chen2017dialogue} requires reward engineering and extensive training. Odin's multi-signal beam search combines the efficiency of informed search with the semantic awareness of learned models.

\section{Limitations and Future Work}

\textbf{Cold Start:} NPLL training takes 10--30s on first connection. Future work includes incremental training as graph evolves.

\textbf{Dynamic Graphs:} Current implementation requires full retrain on significant updates. Streaming PPR~\cite{bahmani2010fast} and incremental EM could address this.

\textbf{Cross-Graph Discovery:} Odin operates on single KGs. Federated exploration across multiple graphs is an open problem.

\textbf{Hypergraph Architecture (V2):} While our GNN bridge scoring mechanism effectively guides exploration across communities, the theoretically "pure" solution for n-ary relationship reasoning is a native hypergraph database where hyperedges can connect arbitrary sets of entities. Current COMPASS scoring augments pairwise PPR with group-level signals; a hypergraph backend would enable direct computation of group-based structural importance without the need for bridge detection. This represents a natural evolution for Odin V2, trading implementation complexity for theoretical elegance.

\textbf{Theoretical Analysis:} Tighter bounds on beam search coverage and NPLL approximation error warrant further investigation.

\section{Conclusion}

We presented Odin, a graph intelligence engine for autonomous discovery in knowledge graphs, introducing the COMPASS score—a novel metric that unifies structural importance (PPR), semantic plausibility (NPLL as discriminative filter), temporal relevance, and community-aware guidance through GNN-identified bridge entities for the first time. The bridge scoring mechanism addresses the "echo chamber" problem in graph exploration, where dense local communities trap traversal algorithms. This principled multi-signal framework enables AI agents to explore large KGs efficiently while maintaining evidence grounding. Production deployments in regulated healthcare and insurance domains demonstrate significant improvements in pattern discovery quality and analyst efficiency. The self-managing NPLL architecture eliminates operational complexity, making Odin practical for enterprise adoption.

As organizations continue to invest in knowledge graphs, the shift from retrieval-based to exploration-based intelligence will become increasingly critical. \textbf{Odin represents the first practical realization of this paradigm}, with production deployments providing empirical evidence that autonomous discovery is not only feasible but transformative. The "compass for agents" architectural pattern (separating graph intelligence from natural language reasoning) provides a reusable template for future agent-graph systems. We anticipate that autonomous discovery will become a standard capability in enterprise knowledge infrastructure, and Odin establishes both the technical foundation and production-readiness pathway for this transition.

\section*{Broader Impact}

While Odin enables discovery in sensitive domains (healthcare, insurance), all deployments operate on de-identified data with strict access controls. The evidence-grounding guarantee prevents hallucination, a key safety property. Future work should investigate fairness implications of automated pattern discovery and ensure discovered patterns do not perpetuate biases present in source data.

\section*{Data and Code Availability}

The knowledge graphs used in this study were constructed from de-identified production data in healthcare and insurance domains under strict data protection agreements. Due to privacy regulations (HIPAA), contractual obligations, and the presence of protected health information (PHI), the underlying datasets cannot be made publicly available.

\textbf{Code Release:} The Odin KG Engine library (Phase 2, online intelligence) will be made available as open-source upon publication, including: (1) COMPASS scoring implementation, (2) beam search with multi-signal guidance, (3) self-managing NPLL training, (4) integration adapters for ArangoDB and Neo4j, and (5) documentation of expected schema for community metadata (optional bridge/affinity tables). Synthetic evaluation scripts for public KG benchmarks (FB15k-237, NELL-995) will also be provided.

The extraction pipeline (Phase 1, offline) is proprietary but follows standard document-to-KG patterns. Community detection can be reproduced using open-source GNN libraries (PyTorch Geometric + GAT layers) with the architectural guidance provided in Section 2.4. Researchers can integrate Odin with their own KG construction pipelines by adhering to the documented metadata schema.

\section*{Author Contributions}

M.K. conceived and designed Odin, implemented the core system, conducted experiments, and led the writing. E.N. engineered the NPLL component, and initial core work on GNN algorithm for community clustering. All authors reviewed and approved the final manuscript.

\section*{Acknowledgments}

The authors thank Dennis Kevogo for reviewing  and providing crucial feedback on hypergraph scaling; and the Prescott Data engineering team and deployment partners for their contributions. This work was supported by production deployments in healthcare and insurance domains.

\bibliographystyle{plain}

\end{document}